\documentclass[letterpaper, 10 pt, journal, twoside]{IEEEtran}
\IEEEoverridecommandlockouts                              % This command is only
\usepackage{cite}
\usepackage{enumerate}
\usepackage[english]{babel}
\usepackage[utf8]{inputenc}
\usepackage{algorithm}
\usepackage[noend]{algpseudocode}
\usepackage{amsmath}
\usepackage{amssymb}
\usepackage{graphicx}
\usepackage{hyperref}
\usepackage{graphics}
\newtheorem{lemma}{Lemma}
\newtheorem{theorem}{Theorem}

\newtheorem{proof}{Proof}

\usepackage[dvipsnames]{xcolor}

\newcommand{\R}{\mathbb{R}}
\newcommand{\E}{\mathbb{E}}
\newcommand{\F}{\mathcal{F}}
\newcommand{\g}{\gamma}

\usepackage{tikz}
\usepackage{textcomp}
\usepackage{hyperref}
\usepackage{lipsum}

\newcommand\copyrighttext{%
  \footnotesize \textcopyright 2019 IEEE. Personal use is permitted, but republication/redistribution requires IEEE permission. This paper is accepted at IEEE Control Systems Letters
  DOI: 10.1109/LCSYS.2019.2921158}
\newcommand\copyrightnotice{%
\begin{tikzpicture}[remember picture,overlay]
\node[anchor=south,yshift=10pt] at (current page.south) {\fbox{\parbox{\dimexpr\textwidth-\fboxsep-\fboxrule\relax}{\copyrighttext}}};
\end{tikzpicture}%
}

\title{
Successive Over-Relaxation Q-Learning
\thanks{This work was supported by Robert Bosch Centre for Cyber-Physical Systems, Indian Institute of Science and a
grant from the Department of Science and Technology, India.}
}

\author{Chandramouli Kamanchi, Raghuram Bharadwaj Diddigi and Shalabh Bhatnagar% <-this % stops a space
\thanks{C. Kamanchi and R. B. Diddigi are with the Department of Computer
Science and Automation, Indian Institute of Science, Bengaluru 560012,
India (e-mail: chandramouli@iisc.ac.in; raghub@iisc.ac.in).}%
\thanks{S. Bhatnagar is with the Department of Computer Science and
Automation, Indian Institute of Science, Bengaluru 560012, India, and
also with the Department of Robert Bosch Centre for Cyber-Physical
Systems, Indian Institute of Science, Bengaluru 560012, India (e-mail:
shalabh@iisc.ac.in).}
}

\begin{document}

\maketitle
\thispagestyle{empty}
\copyrightnotice
%%%%%%%%%%%%%%%%%%%%%%%%%%%%%%%%%%%%%%%%%%%%%%%%%%%%%%%%%%%%%%%%%%%%%%%%%%%%%%%%
\begin{abstract}In a discounted reward Markov Decision Process (MDP), the objective is to find the optimal value function, i.e., the value function corresponding to an optimal policy. This problem reduces to solving a functional equation known as the Bellman equation and a fixed point iteration scheme known as the value iteration is utilized to obtain the solution. In literature, a successive over-relaxation based value iteration scheme is proposed to speed-up the computation of the optimal value function. The speed-up is achieved by constructing a modified Bellman equation that ensures faster convergence to the optimal value function. However, in many practical applications, the model information is not known and we resort to Reinforcement Learning (RL) algorithms to obtain optimal policy and value function. One such popular algorithm is Q-learning. In this paper, we propose Successive Over-Relaxation (SOR) Q-learning. We first derive a modified fixed point iteration for SOR Q-values and utilize stochastic approximation to derive a learning algorithm to compute the optimal value function and an optimal policy. We then prove the almost sure convergence of the SOR Q-learning to SOR Q-values. Finally, through numerical experiments, we show that SOR Q-learning is faster compared to the standard Q-learning algorithm.

\end{abstract}
\section{Introduction}
\IEEEPARstart{I}{n} a discounted reward Markov Decision Process (MDP), the objective is to find optimal value function and a corresponding optimal policy. If the model information is known, the optimal value function can be computed by finding the fixed points of the Bellman equation through value iteration scheme \cite{bertsekas1996neuro}. The contraction factor for this fixed point scheme is seen to be the discount factor of the MDP. It determines the rate of convergence of the value function estimates (obtained from value iteration) to the optimal value function. In \cite{reetz1973solution}, a modified Bellman equation using the concept of SOR is proposed that is shown to have a contraction factor less than or equal to the discount factor of the MDP. More specifically, under a special structure for MDPs, it can be shown that the contraction factor is strictly less than the discount factor. The special structure for the MDP is as follows. For each action in the action space, there is a positive probability of self loop for every state in the state space. 

Reinforcement Learning algorithms are used to obtain the optimal policy and value function when the full model of the MDP is not known. These algorithms make use of the state and reward samples to compute the optimal policy.
One of the popular Reinforcement learning algorithms is the Q-learning algorithm. The Q-learning algorithm combined with the Deep Learning framework has gained popularity in recent times and has been successfully applied to solve many problems \cite{mnih2015human}.  
In this paper, we propose a generalized Q-learning algorithm based on the Successive Over-Relaxation technique. First, we derive a Q-value based modified Bellman operator and show that the contraction factor of this operator is less than or equal to the discount factor. We then utilize the stochastic approximation technique to derive the generalized Q-learning algorithm that we call as SOR Q-learning. 

We now point out some of the variants of the standard Q-learning algorithm in the literature. In \cite{peng1994incremental}, the $Q(\lambda)$ algorithm has been proposed that combines ideas of Q-learning and eligibility traces. In \cite{hasselt2010double}, the Double Q-learning algorithm has been proposed to mitigate the problem of overestimation in Q-learning. Double Q-learning makes use of two Q-value functions in the update equation. In \cite{ghavamzadeh2011speedy}, speedy Q-learning has been proposed for improving the convergence of the Q-estimates. The speedy Q-learning algorithm makes use of the current and the previous Q-value estimates in the update equation. More recently, the zap Q-learning algorithm has been proposed in \cite{devraj2017zap} that imitates the stochastic Newton-Raphson method and it is shown that zap Q-learning exhibits faster convergence to the optimal solution compared to the standard Q-learning algorithm. 

% In our work, we first derive a modified Q-value Bellman equation and then apply stochastic approximation technique to arrive at our SOR Q-learning algorithm. 
Note that unlike \cite{ghavamzadeh2011speedy,hasselt2010double}, our algorithm utilizes only the current Q-value estimates in the update equation and unlike \cite{devraj2017zap}, our algorithm uses only scalar-valued and not matrix-valued step-sizes.
Our key contributions in this paper are as follows:
\begin{itemize}
\item We construct the modified Q-Bellman equation using the SOR technique.
\item We derive a generalized Q-learning algorithm (SOR Q-learning) using an incremental update stochastic approximation scheme.
\item We prove that the contraction factor of the modified Q-Bellman operator is less than or equal to the contraction factor of the Q-Bellman operator. 
\item We show the almost sure convergence of SOR Q-learning iterates to the SOR Q-values.
\item Through numerical evaluation, we demonstrate the effectiveness of our algorithm.
\end{itemize}
The rest of the paper is organized as follows. In Section II, we introduce the necessary background. We propose our algorithm in Section III. Section IV describes the convergence analysis. Section V presents the results of our numerical experiments. Finally Section VI presents concluding remarks and future research directions.

\section{Background and Preliminaries}
A Markov Decision Process (MDP) is defined by a tuple $(S,A,p,r,\alpha)$ where $S:=\{1,2,\cdots,i,j,\cdots,M \}$ is the set of states, $A$ is the finite set of actions, $p$ denotes the transition probability rule i.e., $p(j|i,a)$ denotes the probability of transition to state $j$ from state $i$ when action $a$ is chosen. $r(i,a,j)$ denotes the single-stage reward obtained in state $i$ when action $a$ is chosen and the system transitions to state $j$. Also, $0 \leq \alpha < 1$ denotes the discount factor. The goal of the MDP is to learn an optimal policy i.e., $\pi: S \xrightarrow{} A$, where $\pi(i)$ indicates the action to be taken in state $i$ that maximizes the discounted reward objective:
\begin{align}
    \E \Big[ \sum_{t = 0}^{\infty} \alpha^{t}r(s_{t},\pi(s_{t}),s_{t+1}) \mid s_{0} = i \Big],
\end{align}
where $s_{t}$ is the state of the system at time $t$ and $\E[.]$ is the expectation taken over the states obtained over time $t = 1,\ldots,\infty$. We denote $V(i)$ to be the optimal value function associated with state $i$. It can be shown that the optimal value function is the solution to the Bellman equation \cite{bertsekas1996neuro}:
\begin{align}\label{v-eq}
    V(i)=\max_{a \in A} \Big \{ \sum_{j=1}^{M} p(j|i,a) \big{(}r(i,a,j)+\alpha V(j)\big{)} \Big \}.
\end{align}
Let $\zeta$ denote the set of all bounded functions from $S$ to $\R$. Then equation \eqref{v-eq} can be viewed as a fixed point equation given by:
\begin{align}\label{fp-vi}
    V = TV,
\end{align}
where the operator $T: \zeta \xrightarrow{} \zeta$ is a function given by $$(TV)(i) = \max_{a \in A} \Big\{ r(i,a)+\alpha \displaystyle\sum_{j=1}^{M} p(j|i,a)V(j) \Big\},$$ and 
$r(i,a)=\displaystyle\sum_{j=1}^{M} p(j|i,a)r(i,a,j). $  

Value iteration is a well-known fixed point iteration scheme employed to solve \eqref{fp-vi}. In the value iteration scheme, an initial $V_{0}$ is selected and a sequence of $V_{n}, ~ n \geq 1$ is obtained as follows:
\begin{align}
    V_{n} = TV_{n-1}, ~n \geq 1.
\end{align}
It can be shown that the optimal value function
\begin{align}
    V = \lim_{n \xrightarrow{} \infty} V_{n} = TV.
\end{align}
In this way, we numerically compute the optimal value function when the model information is known. However, in many practical applications, we do not have access to the model information. Instead, the states visited and reward samples are available to us and the objective is to find the optimal value function and a corresponding policy from this information. One of the popular algorithms for computing the optimal policy and value function from samples is Q-learning \cite{watkins1992q}. 

We now briefly discuss the derivation of the Q-learning update rule from the fixed point iteration discussed above. Let $Q(i,a)$ be defined as
\begin{align}\label{ql-eq}
    Q(i,a) := r(i,a)+\alpha \sum_{j=1}^{M} p(j|i,a)V(j).
\end{align}
Here $Q(i,a)$ is the optimal Q-value function associated with state $i$ and action $a$. Then from \eqref{v-eq}, it is clear that
\begin{align}
    V(i) = \max_{a \in A}Q(i,a).
\end{align}
Therefore, the equation \eqref{ql-eq} can be re-written as follows:
\begin{align}\label{ql-eq2}
    Q(i,a) = r(i,a) + \alpha \sum_{j=1}^{M} p(j|i,a) \max_{b \in A}Q(j,b).
\end{align}
This is the Bellman equation involving Q-values $Q(i,a)$ instead of the value function $V$. We obtain the optimal policy by letting
\begin{align}
    \pi(i) = \arg \max_{a \in A}Q(i,a).
\end{align}

%The corresponding optimal value function is then
% \begin{align}
%     V(i) = \max_{a \in A}Q(i,a).
% \end{align}
It is easy to see that the contraction factor for Q-value iteration is $\alpha$, the discount factor \cite{bertsekas1996neuro}. The contraction factor indicates how fast the Q-value estimates converge to the optimal Q-values. 
% We can write
% \begin{align}
%     \| Q_{n+1} - Q\| \leq \alpha \| Q_{n} - Q\|,
% \end{align}
% where $\|.\|$ is the max-norm. 
Finally, the Q-learning update can be obtained from equation \eqref{ql-eq2} by applying the stochastic fixed point iteration scheme \cite{bertsekas1996neuro} as follows:
\begin{align}
    Q_{n+1}(i,a) =&(1 - \g_{n}(i,a)) Q_{n}(i,a) +  \label{st-q}\\  &\g_{n}(i,a) \big(r(i,a,j) + \max_{b \in A}Q_{n}(j,b)\big) \nonumber,
\end{align}
where $Q_{n}(i,a)$ is the current estimate of the Q-values, $\g_{n}(i,a)$ is a diminishing step-size sequence and $(i,a,r,j)$ is the current (state, action, reward, next state) sample. The convergence of Q-learning to the optimal Q-values under reasonably general conditions is established in \cite{watkins1992q}. 
%A comprehensive discussion on MDPs and  Reinforcement Learning algorithms can be found in \cite{bertsekas1996neuro,sutton1998introduction}.
\begin{algorithm}[h!]
\caption{Successive Over-Relaxation Q-Learning}\label{alg:SOR Q-Learning}
\hspace*{\algorithmicindent} \textbf{Input:}\\ 
\hspace*{\algorithmicindent} $w$: Choose $w \in [1,w^*]$ (refer Section IV) is an over- \\ \hspace*{1 cm}relaxation parameter \\
\hspace*{\algorithmicindent} $i_n,a_n,i_{n+1}$ : current state, action and next state \\
\hspace*{\algorithmicindent} $r(i_n,a_n,i_{n+1})$: single-stage reward \\
\hspace*{\algorithmicindent} $Q_{n}(i_n,a_n)$ : current estimate of $Q(i_n,a_n)$ \\
\hspace*{\algorithmicindent} \textbf{Output:} Updated Q-values $Q_{n+1}$ estimated after $n$ \\ \hspace*{1.9cm} iterations of the algorithm
\begin{algorithmic}[1]
\Procedure{SOR Q-Learning:}{}
\State $d_{n+1} = w\Big(r(i_{n},a_{n},i_{n+1})+\alpha\displaystyle\max_{b \in A}Q_n(i_{n+1},b)\Big)$
\hspace*{1.5cm}$+(1-w)\displaystyle\max_{c \in A}Q_n(i_n,c)-Q_n(i_n,a_n)$ \vspace{0.1cm}
\State $Q_{n+1}(i_{n},a_{n}) = Q_{n}(i_{n},a_{n}) + \g_{n}(i_n,a_n)d_{n+1}$
\State \textbf{return} $Q_{n+1}$ 
\EndProcedure
\end{algorithmic}
\end{algorithm}
In this work, we derive a modified Q-Bellman equation that has a contraction factor less than or equal to $\alpha$. To this end, we utilize the Successive Over-Relaxation (SOR) technique proposed in \cite{reetz1973solution} for the optimal value function. We propose our SOR Q-learning algorithm based on the modified Bellman equation involving Q-values. 

% It is well known that value iteration is the procedure that solves for optimal value function. Value iteration is a fixed point iteration scheme. Successive Over Relaxation is a procedure to improve the rate of converge of fixed point schemes. In \cite{reetz1973solution} Successive Over Relaxation based Bellman equation is introduced and is shown to be a faster fixed point iteration. A natural question is "what is the corresponding Q-Bellman equation and the Q-learning algorithm?" We address these question in this work.
\section{Proposed Algorithm}
In this section, we describe our SOR Q-learning algorithm.
We assume that we have a trajectory $\{(i_n,a_n,r(i_n,a_n,i_{n+1}),i_{n+1})\}^{\infty}_{n=1}$ in which each tuple $(i,a) \in  S \times A$ appears infinitely often. At each time step $n$, the input to the algorithm is an over-relaxation parameter $w$, current single-stage reward $r(i_{n},a_{n},i_{n+1})$, the next state $i_{n+1}$ and the current SOR Q-learning estimates $Q_{n}$. The algorithm proceeds to calculate quantities $d_{n+1}$ and $Q_{n+1}$ as given by steps 2 and 3 in Algorithm \ref{alg:SOR Q-Learning}. The procedure terminates after a fixed number of iterations or after a desired accuracy is obtained. Note the difference in the estimation of $d_{n+1}$ between Algorithm \ref{alg:SOR Q-Learning} and standard Q-learning (refer equation \eqref{st-q}). Observe that if $w=1$, SOR Q-learning reduces to the standard Q-learning. Therefore SOR Q-learning can be viewed as a generalization of standard Q-learning. 

\section{Convergence Analysis}
In this section we first prove necessary results that are used in the subsequent analysis.
\vspace{0.2cm}
\begin{lemma}
Suppose we are given finite length sequences $\{a_n\}^{L}_{n=1}$ and $\{b_n\}^{L}_{n=1}$. Then,
$$\big{|}\max_{n}\{a_n\}-\max_{n}\{b_n\}\big{|} \leq \max_{n}\{\big{|}a_{n}-b_{n}\big{|}\}.$$
\label{l1}
\end{lemma}
\begin{proof}
Clearly $\max\{a_n\}\geq a_n$ and $\max\{b_n\}\geq b_n$. So we have \begin{align*}
%  & \max\{a_n\}+\max\{b_n\} \geq a_n+b_n \\ 
%  \implies & 
 \max\{a_n\}+\max\{b_n\} \geq \max{\{a_n+b_n\}}. \nonumber
\end{align*}
Replacing
$\{a_n\}$ by $\{a_n-b_n\}$ we get,
\begin{align*}
    % & \max\{a_n-b_n\}+\max\{b_n\} \geq \max\{a_n\} \\ 
    % \implies 
    &\max\{a_n-b_n\} \geq \max\{a_n\}-\max\{b_n\} \nonumber \\
    & \text{Observe that }\big{|}a_{n}-b_n\big{|} \geq a_n-b_n \nonumber \\ & \implies \max\{\big{|}a_{n}-b_n\big{|}\}  \geq \max\{a_n-b_n\}  \nonumber
    \\ & \hspace{3.3 cm} \geq \max\{a_n\}-\max\{b_n\}. \nonumber
\end{align*}
% $\max\{a_n\}+\max\{b_n\} \geq a_n+b_n.$ Hence $\max\{a_n\}+\max\{b_n\} \geq \max{\{a_n+b_n\}}.$ Replacing
% $\{a_n\}$ by $\{a_n-b_n\}$ we get $\max\{a_n-b_n\}+\max\{b_n\} \geq \max\{a_n\}$ giving us
% $\max\{a_n-b_n\} \geq \max\{a_n\}-\max\{b_n\}.$
% Now $\big{|}a_{n}-b_n\big{|} \geq a_n-b_n.$ So
% $\max\{\big{|}a_{n}-b_n\big{|}\} \geq \max\{a_n-b_n\} \geq \max\{a_n\}-\max\{b_n\}$. 
Similarly we can show that
\begin{align*}
    \max\{a_n\}-\max\{b_n\} \geq - \max\{b_n-a_n\}. 
\end{align*}
Using $\max\{\big{|}a_n-b_n\big{|}\} \geq \max\{b_n-a_n\}$ we get,
\begin{align*}
    -\max\{\big{|}a_n-b_n\big{|}\} & \leq \max\{a_n\}-\max\{b_n\} \\  & \leq \max\{\big{|}a_n-b_n\big{|}\}. \nonumber
\end{align*}
Hence,
\begin{align*}
    \big{|}\max\{a_n\}-\max\{b_n\}\big{|} \leq \max\{\big{|}a_{n}-b_{n}\big{|}\}.
\end{align*}
\end{proof}

Recall that for a given MDP $(S,A,p,r,\alpha)$ the optimal value function $V^{*}$ satisfies \cite{bellman1966dynamic, blackwell1962discrete} the Bellman equation
\begin{equation}
\label{Bellmaneqn}
V^{*}(i)=\max_{a \in A}\Big{\{}r(i,a)+\alpha \sum^{M}_{j=1}p(j|i,a)V^{*}(j)\Big{\}}.
\end{equation}
% Let $T:\R^{|S|} \rightarrow \R^{|S|}$ be defined as follows:
% $$(TV)(i):=\max_{a \in A}\Big{\{}r(i,a)+\alpha \sum^{M}_{j=1}p(j|i,a)V(j)\Big{\}}.$$
It can be seen \cite{denardo1967contraction} that $T$ is a contraction under the max-norm $\|u\|:=\displaystyle\max_{1\leq i\leq M}|u(i)|$ with contraction factor $\alpha$.\\
Let $w^*$ be given by
\begin{align}
    w^*=\displaystyle\min_{i,a}\Bigg{\{}\frac{1}{1-\alpha p(i|i,a)}\Bigg{\}}.
    \label{w-star}
\end{align}
Note that $w^* \geq 1$. For $0 < w \leq w^*$ define a modified operator
$T_{w}:\R^{|S|} \rightarrow \R^{|S|}$ as follows:
$$(T_{w}V)(i)=w ~ TV(i)+(1-w)V(i),$$ where $w$ represents a prescribed relaxation factor.
Observe that the optimal value function $V^*$ is also the unique fixed point of $T_{w}.$ Moreover it is shown \cite{reetz1973solution} that $T_{w}$ is a contraction with contraction factor $\xi(w)$ and $\xi(w^*)\leq \alpha.$
Now we have 
\begin{align*}
& (T_{w}V)(i)= \displaystyle\max_{a \in A}\Big{\{}w\big{(}r(i,a)+\alpha \sum^{M}_{j=1}p(j|i,a)V(j)\big{)} \\ 
& \hspace{4.5cm}+(1-w)V(i)\Big{\}}.
\end{align*}
Let $Q^*(i,a)$ be defined as follows:
\begin{align}
Q^*(i,a):=w\bigg{(}r(i,a)+\alpha \sum^{M}_{j=1}p(j|i,a)V^*(j)\bigg{)} \nonumber \\ 
\hspace{3cm} +(1-w)V^*(i).\label{MQU}
\end{align}
Since $V^*$ is the unique fixed point of $T_{w}$ clearly it can be seen that $$V^*(i)=(T_{w}V^*)(i)=\displaystyle\max_{a \in A} Q^*(i,a) ~ \forall i \in S.$$ Hence the equation \eqref{MQU} can be rewritten as follows:
\begin{align}
\label{Q-star}
& Q^*(i,a)=w\bigg{(}r(i,a)+\alpha \sum^{M}_{j=1}p(j|i,a) \displaystyle\max_{b \in A }Q^*(j,b)\bigg{)} \nonumber \\ 
& \hspace{4cm} +(1-w) \displaystyle\max_{c \in A}Q^*(i,c).
\end{align}
Let $H_w:\R^{|S| \times |A|} \rightarrow \R^{|S| \times |A|}$ be defined as follows.
\begin{align*}
(H_{w}Q)(i,a):=w\bigg{(}r(i,a)+ & \alpha \sum^{M}_{j=1}p(j|i,a) \displaystyle\max_{b \in A }Q(j,b)\bigg{)} \\ 
& +(1-w) \displaystyle\max_{c \in A}Q(i,c).
\end{align*}

\begin{lemma}
$H_w:\R^{|S| \times |A|} \rightarrow \R^{|S| \times |A|}$ is a max-norm contraction and $Q^*$ is the unique fixed point of $H_w$.
\label{l2}
\end{lemma}
\begin{proof}
Observe that $Q^*$ is a fixed point of $H_w$ from equation \eqref{Q-star}.  It is enough to show that $H_w$ is a max-norm contraction.
For $P,Q \in \R^{|S| \times |A|}$, we have
\begin{align}
 &\bigg{|}(H_{w}P-H_{w}Q)(i,a)\bigg{|} \nonumber \\
=&\bigg{|}w \alpha \displaystyle\sum^{M}_{j=1}p(j|i,a)(\displaystyle\max_{b \in A }P(j,b)-\displaystyle\max_{b \in A }Q(j,b)) \nonumber  \\
&\hspace{2.5cm}+(1-w)(\max_{c \in A}P(i,c)-\displaystyle\max_{c \in A}Q(i,c))\bigg{|} \nonumber \\
=&\bigg{|}w \alpha \displaystyle\sum^{M}_{j=1, j\neq i}p(j|i,a)(\displaystyle\max_{b \in A }P(j,b)-\displaystyle\max_{b \in A }Q(j,b))  \nonumber \\
&\hspace{0.45cm}+(1-w+w\alpha p(i|i,a))(\max_{c \in A}P(i,c)-\displaystyle\max_{c \in A}Q(i,c))\bigg{|} \nonumber \\
\end{align}
\begin{align}
\nonumber \\ 
\leq &\bigg{|}w \alpha \displaystyle\sum^{M}_{j=1, j\neq i}p(j|i,a)(\displaystyle\max_{b \in A }P(j,b)-\displaystyle\max_{b \in A }Q(j,b))\bigg{|}  \nonumber \\
&\hspace{0.1cm}+\big{|}(1-w+w\alpha p(i|i,a))\big{|}\bigg{|}(\max_{c \in A}P(i,c)-\displaystyle\max_{c \in A}Q(i,c))\bigg{|} \label{w-con} \\
\leq & w \alpha \displaystyle\sum^{M}_{j=1,j\neq i}p(j|i,a)\bigg{|}\displaystyle\max_{b \in A }P(j,b)-\displaystyle\max_{b \in A }Q(j,b)\bigg{|} \nonumber \\ &\hspace{0.6cm} +(1-w+w\alpha p(i|i,a))\bigg{|}\displaystyle\max_{c \in A }P(i,c)-\displaystyle\max_{c \in A }Q(i,c)\bigg{|} \label{apply-l1} \\
\leq & w \alpha \displaystyle\sum^{M}_{j=1,j\neq i}p(j|i,a)\displaystyle\max_{b \in A }\bigg{|}P(j,b)-Q(j,b)\bigg{|} \nonumber \\ 
&\hspace{1.3cm} +(1-w+w\alpha p(i|i,a))\displaystyle\max_{b \in A }\bigg{|}P(i,b)-Q(i,b)\bigg{|} \label{apply2-l1} \\
\leq & (w \alpha + 1-w)\|P-Q\|.\nonumber \\
&\text{Hence,} \nonumber \\
&\displaystyle \max_{(i,a)}|(H_{w}P-H_{w}Q)(i,a)|\leq (w \alpha +1-w)\|P-Q\|, \nonumber \\
&\text{or }\|(H_{w}P-H_{w}Q)\|\leq (w \alpha+1-w)\|P-Q\|. \nonumber
\end{align}
Note that in equation \eqref{w-con}, we make use of the assumption $0< w\leq w^*$ (refer equation \eqref{w-star}) that ensures the term $\big{(}1-w+w\alpha p(i|i,a)\big{)}\geq0$ to arrive at equation \eqref{apply-l1}. Also note the application of Lemma \ref{l1} in equation \eqref{apply-l1} to arrive at equation \eqref{apply2-l1}. Finally it is easy to see from the assumptions on $w$ and discount factor $\alpha$ that $0< (w\alpha+1-w)<1$. 
Therefore $H_w$ is a max-norm contraction with contraction factor 
$(w\alpha+1-w)$ and $Q^*$ is the unique fixed point.
\vspace{0.15cm}
\end{proof}
\begin{lemma}
Let $Q$ be the solution of the standard Q-learning algorithm and $Q^*$ be the fixed point of $H_{w}$. Then for all $i \in S$ $Q(i,c)=Q^*(i,c)$ where $c=\displaystyle\arg\max_{b\in A}Q(i,b)$, is an optimal action in state $i$.
\label{l3}
\end{lemma}
\vspace{0.2cm}
\begin{proof}
Since $Q$ is the solution obtained by the standard Q-learning algorithm, $Q$ is the fixed point of $H$ given by
$(HQ)(i,a)=\bigg{(}r(i,a)+ \alpha \sum^{M}_{j=1}p(j|i,a) \displaystyle\max_{b \in A }Q(j,b)\bigg{)}$ (refer equation \eqref{ql-eq2}) i.e. $HQ=Q$. Since $Q^*$ is a fixed point of $H_{w}$, we have
\begin{align*}
      & Q^*(i,c)=(H_{w}Q^*)(i,c)\\ 
    = & w\bigg{(}r(i,c)+\alpha \sum^{M}_{j=1}p(j|i,c)V^*(j)\bigg{)} \\
    & \hspace{3cm} +(1-w)V^*(i) ~ (\text{from equation} ~ \eqref{MQU})\\ 
%  = & w\bigg{(}r(i,c)+\alpha \sum^{M}_{j=1}p(j|i,c)V(j)\bigg{)}+(1-w)V(i)\\
 = & w Q(i,c)+ (1-w)Q(i,c)  ~ (\text{Since } V^*(i)=Q(i,c))\\
 = & Q(i,c).
\end{align*} Therefore
$$Q^*(i,c)=Q(i,c)$$
for all $(i,c)$, where $c$ is an optimal action in state $i$.
\end{proof}
The above result shows that SOR Q-learning algorithm computes the optimal value function.
\vspace{0.2cm}
\begin{lemma}
For $1\leq w \leq w^*$ the contraction factor for the map $H_w,$
$$1-w+\alpha w\leq \alpha.$$
\label{l4}
\end{lemma}
\begin{proof}
Define $f(w)=1-w+\alpha w.$ Clearly $f'(w)=-(1-\alpha)<0.$ Hence $f$ is decreasing. Also observe that $f(1)=\alpha.$ Hence for $1\leq w \leq w^*$, $1-w+\alpha w=f(w)\leq f(1)=\alpha.$
\end{proof}
This is one of the key results in this paper. This lemma shows that the SOR Q-learning iterates asymptotically track the optimal Q-values faster than the standard Q-learning.
% This lemma shows that the SOR Q-value iteration scheme will converge faster than the regular Q-value iteration.

We apply the following theorem \cite{ jaakkola1994convergence} to show the convergence of the iterates of SOR Q-learning to the optimal Q-values.
\vspace{0.15cm}
\begin{theorem}
The $p$-dimensional random process $\{\Delta_n\}$ taking values in $\R^{p}$ and defined as 
$$\Delta_{n+1}(l)=\big{(}1-\g_n(l)\big{)}\Delta_{n}(l)+\g_n(l)F_n(l), 1\leq l \leq p,$$
converges to zero with probability 1 as $n \rightarrow \infty$ under the following assumptions:
\begin{itemize}
    \item $0\leq \g_n(l) \leq 1, \displaystyle\sum^{\infty}_{n=1} \g_{n}(l)=\infty$ and $\displaystyle\sum^{\infty}_{n=1} \g^2_{n}(l)<\infty$;\\
    \item $\bigg{\|}\E\big{[}F_n \mid \F_n\big{]}\bigg{\|} \leq \beta \big{\|}\Delta_n\big{\|},$ with $\beta <1$;\\
    \item $\textbf{var}\big{[}F_n(l)|\F_n\big{]}\leq C\big{(}1+\big{\|}\Delta_n\big{\|}^2\big{)},$ for $C>0,$ \\
\end{itemize}
where $\F_n=\sigma\{\Delta_n, \Delta_{n-1}, \cdots, F_{n-1}, \cdots, \g_n, \cdots \}$ is the $\sigma$-field generated by the quantities inside \{.\}.
\end{theorem}

% Consider the following iterative procedure
% $$r_{n+1}=(1-a_n)r_{n}+a_n(Hr_{n}+N_{n})$$ where $N_{n}$ is a random noise term. Let $\F_{n}$ be defined as
% $\F_{n}=\{r_0,\cdots,r_n,N_0,\cdots,N_{n-1},a_0,\cdots,a_n\}$
% \begin{theorem}
% Suppose 
% $$\E[N_n|\F_n]=0$$
% Given any norm $\|.\|$ on $\R^{k}$, there exist constants $A$ and $B$ such that
% $$\E[\|N_n\|^2]\leq A+B\|r_n\|^2$$
% The step sizes $a_n$ are non-negative and satisfy 
% $$\displaystyle\sum^{\infty}_{n=0}a_{n}=\infty,  \displaystyle\sum^{\infty}_{n=0}a^{2}_{n}<\infty.$$
% The mapping $H$ is a maximum norm contraction. Then $r_n$ converges with probability 1 to the fixed point of $H$
% \end{theorem}
\vspace{0.2cm}
\begin{theorem}
Given a finite MDP $(S,A,p,r,\alpha)$ with bounded rewards i.e. $|r(i,a,j)| \leq B<\infty, ~ \forall ~ (i,a,j) \in S\times A \times S $, the SOR Q-learning algorithm given by the update rule:
\begin{align*}
& Q_{n+1}(i,a) = Q_{n}(i,a) + \g_{n}(i,a)\Bigg{(}w \Big{(}r(i,a,j) \\ 
&\hspace{0.5cm}+\alpha \max_{b \in A}Q_n(j,b)\Big{)}+(1-w)\max_{c \in A}Q_n(i,c)-Q_n(i,a)\Bigg{)}
\end{align*}
converges with probability 1 to the optimal Q-values as long as 
$$\sum_n \g_n(i,a)= \infty, \hspace{1cm} \sum_n \g^2_n(i,a)<\infty,$$
for all $(i,a) \in S\times A$.
\end{theorem}
\begin{proof}
Upon rewriting the update rule we have
\begin{align*}
    & Q_{n+1}(i,a)=\big(1-\g_n(i,a)\big)Q_{n}(i,a) \\ & \hspace{1cm}+\g_n(i,a)\Big{[}w\big{(}r(i,a,j)+\alpha \max_{b\in A}Q_{n}(i,b)\big{)} \\
    & \hspace{2cm}+(1-w)\max_{c \in A}Q_{n}(i,c)\Big{]}.
\end{align*}
Define $\Delta_{n}(i,a)=Q_{n}(i,a)-Q^*(i,a)$ and 
$\F_n=\sigma\big{(}\{Q_{0},\g_{j}, i_j, a_j \forall j \leq n, ~ n \geq 0\}\big{)}$ be the filtration.
The update rule of the algorithm can be written as
\begin{align*}
& \Delta_{n+1}(i_{n},a_{n})= \big(1-\g_n(i_{n},a_{n})\big)\Delta_{n}(i_{n},a_{n}) \\ 
& \hspace{0.6cm} +\g_n(i_{n},a_{n})\big{[}w \big(r_n+\alpha \max_{b\in A}Q_{n}(i_{n+1},b)-Q^*(i_{n},a_n)\big) \\ 
& \hspace{2cm}+(1-w)\big(\max_{c \in A}Q_{n}(i_{n},c)-Q^*(i_{n},a_n)\big)\big{]},
\end{align*}
where $r_n=r(i_{n},a_{n},i_{n+1})$. Let 
\begin{align*}
    & F_n(i,a)=w\big{(}r(i,a,\eta_{i,a})+\alpha \max_{b\in A}Q_{n}(\eta_{i,a},b)-Q^*(i,a)\big{)} \\ 
    & \hspace{1cm} +(1-w)\big(\max_{c \in A}Q_{n}(i,c)-Q^*(i,a)\big),
\end{align*}
where $\eta_{i,a}$ is a random variable having the distribution $p(. \mid i,a).$ We have
\begin{align*}
    &\E\big{[}F_{n}(i,a)|\F_{n}\big{]}\\ &=\displaystyle\sum_{j=1}^{M}p(j|i,a) \bigg[w(r(i,a,j)+\alpha \max_{b\in A}Q_{n}(j,b)-Q^*(i,a)) \\
    & \hspace{3cm} +(1-w)\big(\max_{c \in A}Q_{n}(i,c)-Q^*(i,a)\big)\bigg] \\
   &=(H_{w}Q_n)(i,a)-Q^*(i,a).
\end{align*}
Since $Q^*=H_{w}Q^*$ from Lemma \ref{l2}, we have
\begin{align*}
    \bigg{|}\E\big{[}F_{n}(i,a) \mid \F_{n}\big{]}\bigg{|} & \leq (\alpha w + 1-w)\|Q_{n}-Q^*\| \\
    & = (\alpha w + 1-w) \|\Delta_{n}\|.
\end{align*}
Finally,
\begin{align*}
  & \textbf{var} \big[F_n(i,a) \mid \F_{n}\big] \\
= ~ &\E \Bigg[ \bigg(w \big(r(i,a,\eta_{i,a})+\alpha \max_{b\in A}Q_{n}(\eta_{i,a},b)-Q^*(i,a)\big)\\
& \hspace{1cm}+(1-w) \big(\max_{c \in A}Q_{n}(i,c)-Q^*(i,a)\big) \\
& \hspace{2cm}-H_{w}Q_{n}(i,a)+Q^*(i,a)\bigg)^2 \Bigg]\\
= ~ &\E \Bigg[\bigg(w \big(r(i,a,\eta_{i,a})+\alpha \max_{b\in A}Q_{n}(\eta_{i,a},b)\big)\\
& \hspace{0.5cm}+(1-w) \big(\max_{c \in A}Q_{n}(i,c) \big)- H_{w}Q_{n}(i,a)\bigg)^2 \Bigg]\\
\leq ~ &\E \Bigg[ \bigg(w \big(r(i,a,\eta_{i,a})+\alpha \max_{b\in A}Q_{n}(\eta_{i,a},b)\big)\\
&\hspace{0.5cm}+(1-w) \big(\max_{c \in A}Q_{n}(i,c)\big)\bigg)^2 \Bigg] \\
\leq ~ & 3\bigg{(}w^2B^2+\alpha^2w^2\|Q_n\|^2+(1-w)^2\|Q_n\|^2\bigg{)}\\
\leq ~ & 3\bigg{(}w^2B^2+2 \big{(} \alpha^2w^2+(1-w)^2 \big{)} \big{(} \|Q^*\|^2+\|\Delta_n\|^2 \big{)} \bigg{)}\\
\leq ~ & C(1+\|\Delta_n\|^2),
\end{align*}
where $C=\max \bigg{\{}3w^2B^2+6 \big{(} \alpha^2w^2+(1-w)^2 \big{)} \|Q^*\|^2,6\big{(} \alpha^2w^2+(1-w)^2 \big{)} \bigg{\}}.$
Here the first inequality follows from the fact:
\begin{align*}
\E[Z-\E Z]^2=\E[Z^2]-\E[Z]^2 \leq \E[Z^2].    
\end{align*}
The second inequality follows from:
 \begin{align*}
     & |r(i,a,j)| \leq B, \\
     & \|v\|= \max_{i}{|v(i)|},\\
     & (a+b+c)^2\leq3(a^2+b^2+c^2) ~ \forall a,b,c.
 \end{align*} 
The third inequality from the properties:
 \begin{align*}
  & \forall a,b, ~ (a+b)^2\leq 2(a^2+b^2), \\
  & \text{as well as the triangle inequality of the norm. Hence,}
 \end{align*} 
$$\textbf{var}[F_n(i,a) \mid \F_n]\leq C(1+\|\Delta_n\|^2).$$ 
Therefore by Theorem 1, $\Delta_{n}$ converges to zero with probability 1 i.e., SOR Q-learning iterates $Q_n$ converges to $Q^*$ almost surely.
\end{proof}
\vspace{0.5cm}

\section{Experiments and Results}
In this section, we present the experimental evaluation of our proposed algorithm. First we numerically establish the convergence of our algorithm to the optimal value function. Next, we show the comparison between SOR Q-learning and standard Q-learning when we select the optimal $w^{*}$ (refer Section IV). Finally, we show the comparison between various feasible $w$ values that can be used in our algorithm. 

For our experiments, we construct $100$ independent and random MDPs with $10$ states and $5$ actions that satisfy the assumption i.e. $p(i|i,a)>0, ~ \forall ~ (i,a)$. Note that this condition makes sure that $w^{*} >1$ which in turn ensures that the contraction factor of $H_{w}$ is strictly less than $\alpha$ (refer Lemma \ref{l4}). However any $0 < w \leq w^*$ ensures convergence of SOR Q-learning algorithm.  We use python MDP toolbox \cite{pymdp} to generate the MDPs. For both SOR Q-learning and standard Q-learning algorithms, we maintain the same step-size and run the algorithms for the same number of iterations. Implementation of our SOR Q-learning is available here \footnote{\url{https://github.com/raghudiddigi/SOR-Q-Learning}}.

In Figure \ref{exp1-fig}, we plot the average error as a function of number of iterations. Average error is calculated as follows. For each of 100 runs, we collect the error between the optimal value function and the Q-value estimate at every iteration. Then, the average error is calculated as the mean of the errors collected, i.e., average error at iteration $k$ is 
\begin{align}\label{comp-eqn}
    e(k)=\frac{1}{100}\sum_{m=1}^{100} || V_{m}^{*} - \max_{a}Q_{m}^{k}(.,a)||,
\end{align}
where $V_m^{*}$ is the optimal value function of the $m^{th}$ MDP and $Q_{m}^{k}$ is the Q-value estimate of the $m^{th}$ MDP at iteration $k$.
We can see that, in Figure \ref{exp1-fig}, $e(k)$ decreases as the number of iterations increase. 

\begin{figure}
    \centering
    \includegraphics[scale = 0.4]{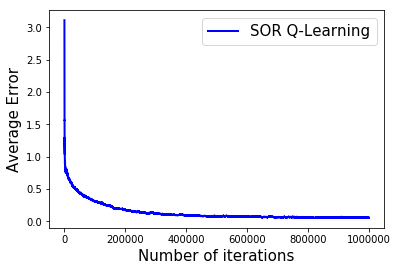}
    \caption{Convergence of SOR Q-learning}
    \label{exp1-fig}
\end{figure}

In Figure \ref{exp2-fig}, we show the comparison between SOR Q-learning and the standard Q-learning over $10^{5}$ iterations. In this experiment, we select optimal $w^{*}$ for our SOR Q-learning. We can see that the average error for SOR Q-learning is less than that of standard Q-learning during the learning process. In Table \ref{exp2-table}, we show the performance of the converged policies in both the cases. Here average error is $e(10^5)$ (refer equation \eqref{comp-eqn}) and average policy difference is computed as the mean of the difference between converged policy and optimal policy for these $100$ MDPs. We observe that, on average, SOR Q-learning gives lower error and a better policy compared to the standard Q-learning. %This shows that, in most of the cases, SOR Q-learning gives better performance than standard Q-learning numerically. 

\begin{figure}
    \centering
    \includegraphics[scale = 0.4]{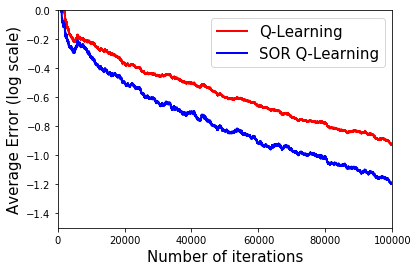}
    \caption{Performance of algorithms as learning progress}
    \label{exp2-fig}
\end{figure}

\begin{table}[H]
\begin{center}
\begin{tabular}{|c|c|c|}
\hline
\textbf{Algorithm} & \textbf{\begin{tabular}[c]{@{}c@{}}Average \\ Error\end{tabular}} & \textbf{\begin{tabular}[c]{@{}c@{}}Average \\ Policy Difference\end{tabular}} \\ \hline
SOR Q-Learning     & 0.3032                                                                      & 0.95                                                                            \\ \hline
Q-Learning         & 0.3962                                                                      & 0.97                                                                            \\ \hline
\end{tabular}
\end{center}
\caption{Performance of converged policies}
\label{exp2-table}
\end{table}
Note that in the above experiment, we have selected optimal $w^{*}$ in SOR Q-learning. However, in Section IV, we showed that any $w$ satisfying $1 < w \leq w^{*}$ would suffice for faster convergence than standard Q-learning. In Figure \ref{exp3-fig}, we show the performance of SOR Q-learning for different $w$ values.
%Note that $w =1$ corresponds to the standard Q-learning algorithm.
We can see that the performance improves as $w$ increases from $1$ to $w^{*}$. Note that any feasible value  of $w > 1$ performs better than $w = 1$ case, which corresponds to the Q-learning algorithm.

\begin{figure}[h!]
    \centering
    \includegraphics[scale = 0.4]{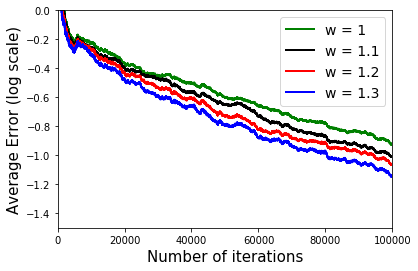}
    \caption{Performance of SOR Q-learning for different $w$-values}
    \label{exp3-fig}
\end{figure}

% \begin{figure}[!htb]
% \begin{figure}[]
% \minipage{0.32\textwidth}
%   \includegraphics[width=\linewidth]{exp1_updated.png}
%   \caption{Convergence of SOR Q-learning}\label{exp1-fig}
% \endminipage\hfill
% \minipage{0.32\textwidth}
%   \includegraphics[width=\linewidth]{exp2_updated.png}
%   \caption{Performance of algorithms as learning progress}\label{exp2-fig}
% \endminipage\hfill
% \minipage{0.32\textwidth}%
%   \includegraphics[width=\linewidth]{exp3_updated.png}
%   \caption{Comparison with different $w$ values in SOR Q-learning}\label{exp3-fig}
% \endminipage
% \end{figure}

\section{Conclusions and Future Work}
In this work, we proposed SOR Q-learning, a generalization of Q-learning that makes use of the concept of Successive Over-Relaxation. We showed that the contraction factor of SOR Q-Bellman operator is less than or equal to $\alpha$, which is the contraction factor of standard Q-Bellman operator. We then proved the convergence of SOR Q-learning iterates to the SOR Q-values. Finally, we numerically established that, on average, SOR Q-learning learns the optimal value function faster than standard Q-learning. 
In future, we would like to extend the concept of SOR to the average cost and risk-sensitive MDPs \cite{huang2017risk}. As with the Q-learning algorithm \cite{even2003learning} it would also be interesting to derive the rate of convergence of SOR Q-learning. %Another line of research would be to develop function approximation versions of SOR Q-learning.
% Subsequently, this can be combined with deep learning architectures and applied in many practical problems. 

%%%%%%%%%%%%%%%%%%%%%%%%%%%%%%%%%%%%%%%%%%%%%%%%%%%%%%%%%%%%%%%%%%%%%%%%%%%%%%%%

\bibliographystyle{IEEEtran}
\bibliography{references}

\end{document}